\documentclass{article}

\usepackage[final, nonatbib]{neurips_2020}

\usepackage[utf8]{inputenc} 
\usepackage[T1]{fontenc}    
\usepackage{hyperref}       
\usepackage{url}            
\usepackage{booktabs}       
\usepackage{amsfonts}       
\usepackage{nicefrac}       
\usepackage{microtype}      

\usepackage{wrapfig}
\usepackage{algorithm} 
\usepackage{algorithmic}
\usepackage{graphicx}
\usepackage{subfigure}
\usepackage{booktabs} 
\usepackage{mathtools}
\usepackage{amssymb}
\usepackage{amsmath}
\usepackage{amsthm}
\usepackage{thmtools,thm-restate}
\usepackage{multirow}
\usepackage{subfigure}
\usepackage{hyperref}

\theoremstyle{definition}

\newtheorem{lemma}{Lemma}[section]

\newcommand*{\dif}{\mathop{}\!\mathrm{d}}
\newcommand{\ie}{\textit{i}.\textit{e}.}
\newcommand{\eg}{\textit{e}.\textit{g}.}
\newcommand{\FF}{\mathcal{F}}
\newcommand{\sspace}{\mathcal{S}}
\newcommand{\aspace}{\mathcal{A}}
\newcommand{\RR}{\mathbb{R}}
\newcommand{\defeq}{\vcentcolon=}

\newcommand{\dtv}{d_{\text{TV}}}

\makeatletter
\newcommand{\printfnsymbol}[1]{%
  \textsuperscript{\@fnsymbol{#1}}%
}
\makeatother

\title{Model-based Policy Optimization with \\ Unsupervised Model Adaptation}

\author{%
  Jian Shen$^\ddagger$, Han Zhao$^{\star}$\thanks{Work done while at Carnegie Mellon University. $^\dagger$Weinan Zhang is the corresponding author.}~, Weinan Zhang$^{\ddagger \dagger}$ , Yong Yu$^\ddagger$ \\
  $^\ddagger$Shanghai Jiao Tong  University, $^{\star}$D. E. Shaw \& Co\\
  \texttt{\{rockyshen, wnzhang, yyu\}@apex.sjtu.edu.cn} \\
  \texttt{han.zhao@cs.cmu.edu} \\
}

\begin{document}

\maketitle

\begin{abstract}
Model-based reinforcement learning methods learn a dynamics model with real data sampled from the environment and leverage it to generate simulated data to derive an agent. However, due to the potential distribution mismatch between simulated data and real data, this could lead to degraded performance. Despite much effort being devoted to reducing this distribution mismatch, existing methods fail to solve it explicitly. In this paper, we investigate how to bridge the gap between real and simulated data due to inaccurate model estimation for better policy optimization. To begin with, we first derive a lower bound of the expected return, which naturally inspires a bound maximization algorithm by aligning the simulated and real data distributions. To this end, we propose a novel model-based reinforcement learning framework AMPO, which introduces unsupervised model adaptation to minimize the integral probability metric (IPM) between feature distributions from real and simulated data. Instantiating our framework with Wasserstein-1 distance gives a practical model-based approach. Empirically, our approach achieves state-of-the-art performance in terms of sample efficiency on a range of continuous control benchmark tasks.
\end{abstract}

\section{Introduction}
\label{sec:intro}
In recent years, model-free reinforcement learning (MFRL) has achieved tremendous success on a wide range of simulated domains, \eg, video games~\cite{mnih2015human}, complex robotic tasks~\cite{sac}, just to name a few. However, model-free methods are notoriously data inefficient and often require a massive number of samples from the environment. In many high-stakes real-world applications, \eg, autonomous driving, online education, etc., it is often expensive, or even infeasible, to collect such large-scale datasets. On the other hand, model-based reinforcement learning (MBRL), in contrast, is considered to be an appealing alternative that is able to substantially reduce sample complexity~\cite{sun2018model, mbbl}.

At a colloquial level, model-based approaches build a predictive model of the environment dynamics and generate simulated rollouts from it to derive a policy~\cite{mbpo,slbo, simple} or a planner~\cite{pets, hafner2018learning}. However, the asymptotic performance of MBRL methods often lags behind their model-free counterparts, mainly due to the fact that the model learned from finite data can still be far away from the underlying dynamics of the environment. To be precise, even equipped with a high-capacity model, such model error still exists due to the potential distribution mismatch between the training and generating phases, \ie, the state-action input distribution used to train the model is different from the one generated by the model~\cite{talvitie2014model}. Because of this gap, the learned model may give inaccurate predictions on simulated data and the errors can compound for multi-step rollouts~\cite{asadi2018lipschitz}, which will be exploited by the follow-up policy optimization or planning procedure, leading to degraded performance.

In the literature, there is a fruitful line of works focusing on reducing the distribution mismatch problem by improving the approximation accuracy of model learning, or by designing careful strategies when using the model for simulation. For model learning, different architectures~\cite{asadi2018lipschitz, asadi2019combating, pets} and loss functions \cite{farahmand2018iterative, mi} have been proposed to mitigate overfitting or improve multi-step predictions so that the simulated data generated by the model are more like real. For model usage, delicate rollout schemes \cite{mbpo, steve, nguyenimproving, xiao2019learning} have been adopted to exploit the model before the simulated data departure from the real distribution.
Although these existing methods help alleviate the distribution mismatch, this problem still exists. 

In this paper, we take a step further towards the goal of explicit mitigation of the distribution mismatch problem for better policy optimization in Dyna-style MBRL \cite{dyna}. To begin with, we derive a lower bound of the expected return in the real environment, which naturally inspires a bound maximization algorithm according to the theory of unsupervised domain adaptation. To this end, we propose a novel model-based framework, namely AMPO (Adaptation augmented Model-based Policy Optimization), by introducing a model adaptation procedure upon the existing MBPO \cite{mbpo} method. To be specific, model adaptation encourages the model to learn invariant feature representations by minimizing integral probability metric (IPM) between the feature distributions of real data and simulated data. By instantiating our framework with Wasserstein-1 distance \cite{ot}, we obtain a practical method. We evaluate our method on challenging continuous control benchmark tasks, and the experimental results demonstrate that the proposed AMPO achieves better performance against state-of-the-art MBRL and MFRL methods in terms of sample efficiency.

\section{Preliminaries}
We first introduce the notation used throughout the paper and briefly discuss the problem setup of reinforcement learning and concepts related to integral probability metric. 
\paragraph{Reinforcement Learning}
A Markov decision process (MDP) is defined by the tuple ($\mathcal{S}, \mathcal{A}, T, r, \gamma$), where $\mathcal{S}$ and $\mathcal{A}$ are the state and action spaces, respectively. Throughout the paper, we assume that the state space is continuous and compact. $\gamma \in (0,1)$ is the discount factor. $T(s'\mid s,a)$ is the transition density of state $s'$ given action $a$ made under state $s$, and the reward function is denoted as $r(s,a)$. The goal of reinforcement learning (RL) is to find the optimal policy $\pi^*$ that maximizes the expected return (sum of discounted rewards), denoted by $\eta$:
\begin{equation}
\label{eq: rl-obj}
\pi^* \defeq \mathop{\arg \max}_\pi \eta[\pi]=\mathop{\arg \max}_\pi \mathbb{E}_\pi \left[\sum_{t=0}^\infty \gamma^t r(s_t,a_t) \right],
\end{equation}
where $s_{t+1} \sim T(s\mid s_t,a_t)$ and $a_t \sim \pi(a\mid s_t)$. In practice, the groundtruth transition $T$ is unknown and MBRL methods aim to construct a model $\hat{T}$ of the transition dynamics, using data collected from interaction with the MDP. Furthermore, different from several previous MBRL works \cite{pets, slbo}, the reward function $r(s,a)$ is also unknown throughout the paper, and an agent needs to learn the reward function simultaneously. 

For a policy $\pi$, we define the normalized occupancy measure, $\rho_{\hat{T}}^\pi(s,a)$ \cite{gail}, as the discounted distribution of the states and actions visited by the policy $\pi$ on the dynamics model $\hat{T}$:
$\rho_{\hat{T}}^{\pi}(s,a)=(1-\gamma)\cdot\pi(a\mid s)\sum_{t=0}^\infty \gamma^t P_{\hat{T},t}^{\pi}(s),$
where $P_{\hat{T},t}^{\pi}(s)$ denotes the density of state $s$ visited by $\pi$ under $\hat{T}$ at time step $t$. Similarly, $\rho_{T}^\pi(s,a)$ represents the discounted occupancy measure visited by $\pi$ under the real dynamics $T$. Using this definition, we can equivalently express the  objective function as follows: 
$
\eta[\pi]=\mathbb{E}_{ \rho_{T}^\pi(s,a)}[r(s,a)]=\int \rho_{T}^\pi(s,a) r(s,a) \dif s \dif a.    $
To simplify the notation, we also define the normalized state visit distribution as $\nu_{T}^\pi(s)\defeq (1-\gamma)\sum_{t=0}^\infty \gamma^t P_{T,t}^\pi(s)$.

\paragraph{Integral Probability Metric}
Integral probability metric (IPM) is a family of discrepancy measures between two distributions over the same space~\cite{ipm, sriperumbudur2009integral}. 
Specifically, given two probability distributions $\mathbb{P}$ and $\mathbb{Q}$ over $\mathcal{X}$, the $\FF$-IPM is defined as 
\begin{equation}
\label{eq:ipm}
d_\mathcal{F}(\mathbb{P}, \mathbb{Q}) \defeq \sup_{f\in\mathcal{F}}~\mathbb{E}_{x\sim \mathbb{P}}[f(x)] - \mathbb{E}_{x\sim \mathbb{Q}}[f(x)],
\end{equation}
where $\mathcal{F}$ is a class of witness functions $f:\mathcal{X} \rightarrow \mathbb{R}$. By choosing different function class $\FF$, IPM reduces to many well-known distance metrics between probability distributions. 
In particular, the Wasserstein-1 distance \cite{ot} is defined using the 1-Lipschitz functions $ \{f: \left \| f \right\|_L \leq 1 \}$, where the Lipschitz semi-norm $\|\cdot\|_L$ is defined as $\left \| f \right\|_L = \sup_{x\neq y} |f(x)-f(y)| / |x-y|$.  
Furthermore, total variation is also a kind of IPM and we use $\dtv(\cdot,\cdot)$ to denote it. 

\section{A Lower Bound for Expected Return}
\label{sec:theory}

In this section, we derive a lower bound for the expected return function in the context of deep MBRL with continuous states and non-linear stochastic dynamics. The lower bound concerns about the expected return, \ie, Eq.~\eqref{eq: rl-obj} and is expressed in the following form \cite{mbpo}:
\begin{equation}
\eta[\pi] \geq \hat{\eta}[\pi] - C,
\end{equation}
where $\hat{\eta}[\pi]$ denotes the expected return of running the policy $\pi$ on a learned dynamics model $\hat{T}(s'\mid s, a)$ and the term $C$ is what we wish to construct. Normally, the dynamics model $\hat{T}$ is learned with experiences $(s,a,s')$ collected by a behavioral policy $\pi_D$ in the real environment dynamics $T$. Typically, in an online MBRL method with iterative policy optimization, the behavioral policy $\pi_D$ represents a collection of past policies. Once we have derived this lower bound, we can naturally design a model-based framework to optimize the RL objective by maximizing the lower bound. Due to page limit, we defer all the proofs to the appendix.

Recall that in MBRL, we have real data $(s,a,s')$ collected in the real dynamics $T$ by the behavioral policy $\pi_D$ and will generate simulated data using the dynamics model $\hat{T}$ with the current policy $\pi$. We begin by showing that for any state $s'$, the discrepancy between its visit distributions in real data and simulated data admits the following decomposition.
\begin{restatable}{lemma}{decomposition}
	\label{lemma}
	Assume the initial state distributions of the real dynamics $T$ and the dynamics model $\hat{T}$ are the same. For any state $s'$, assume there exists a witness function class $\mathcal{F}_{s'} = \{f:\sspace\times \aspace\to \RR\}$ such that $\hat{T}(s'\mid \cdot, \cdot):\sspace\times\aspace\to\RR$ is in $\mathcal{F}_{s'}$. Then the following holds:
		\begin{equation}
		|\nu^{\pi_D}_{T}(s') - \nu^\pi_{\hat{T}}(s')| \leq \gamma d_{\mathcal{F}_{s'}}(\rho_{T}^{\pi_D}, \rho_{\hat{T}}^{\pi}) + \gamma \mathbb{E}_{(s,a) \sim \rho_{T}^{\pi_D}} \left|T(s'\mid s,a) -\hat{T}(s'\mid s,a)\right|. 
		\end{equation}
\end{restatable}
Lemma \ref{lemma} states that the discrepancy between two state visit distributions for each state is upper bounded by the dynamics model error for predicting this state and the discrepancy between the two state-action occupancy measures. Intuitively, it means that when both the input state-action distributions and the conditional dynamics distributions are close then the output state distributions will be close as well. Based on this lemma, now we derive the main result that gives a lower bound for the expected return.

\begin{restatable}{theorem}{main}
	\label{theorem}
	Let $R\defeq \sup_{s,a}r(s,a) < \infty$, $\FF\defeq \cup_{s'\in\sspace}\FF_{s'}$ and define $\epsilon_\pi \defeq 2 \dtv( \nu^{\pi}_{T}, \nu^{\pi_D}_{T})$. Under the assumption of Lemma \ref{lemma}, the expected return $\eta[\pi]$ admits the following bound:
		\begin{equation}	
		\eta[\pi]  \geq \hat{\eta}[\pi] - R \cdot \epsilon_\pi - \gamma R \cdot d_{\FF}(\rho_{T}^{\pi_D}, \rho_{\hat{T}}^{\pi}) \cdot \text{Vol}(\sspace) -\gamma R\cdot \mathbb{E}_{(s,a) \sim \rho_{T}^{\pi_D}}\sqrt{2 D_{\text{KL}}(T(\cdot|s,a)~\|~\hat{T}(\cdot|s,a))},
		\end{equation}
	where $\text{Vol}(\sspace)$ is the volume of state space $\sspace$.
\end{restatable}

\textbf{Remark}~~Theorem \ref{theorem} gives a lower bound on the objective in the true environment. 
In this bound, the last term corresponds to the model estimation error on real data, since the Kullback–Leibler divergence measures the average quality of current model estimation. The second term denotes the divergence between state visit distributions induced by the policy $\pi$ and the behavioral policy $\pi_D$ in the environment, which is an important objective in batch reinforcement learning \cite{fujimoto2018off} for reliable exploitation of off-policy samples. The third term is the integral probability metric between the $(s,a)$ distributions $\rho_{T}^{\pi_D}$ and $\rho_{\hat{T}}^{\pi}$, which exactly corresponds to the distribution mismatch problem between model learning and model usage.

We would like to maximize the lower bound in Theorem \ref{theorem} jointly over the policy and the dynamics model. In practice, we usually omit model optimization in the first term $\hat{\eta}[\pi]$ for simplicity like in previous work \cite{slbo}. Then optimizing the first term only over the policy and the last term over the model together becomes the standard principle of Dyna-style MBRL approaches. And RL usually encourages the agent to explore, so we won't constrain the policy according to the second term since it violates the rule of exploration, which aims at seeking out novel states.
Then the key is to minimize the third term, \ie, the occupancy measure divergence, which is intuitively reasonable since the dynamics model will predict simulated $(s,a)$ samples close to its training data with high accuracy. To optimize this term over the policy, we can use imitation learning methods on the dynamics model, such as GAIL, \cite{gail} where the real samples are viewed as expert demonstrations. However, optimizing this term over the policy is unnecessary, which may further reduce the efficiency of the whole training process. For example, one does not need to further optimize the policy using this term but just uses the $\hat{\eta}[\pi]$ term when the model is sufficiently accurate. So in this paper, we mainly focus on how to optimize this occupancy measure matching term over the model.

\section{AMPO Framework}

\begin{algorithm}[t]
			\caption{AMPO}
			\label{alg}
			\begin{algorithmic}[1]
				\STATE Initialize policy $\pi_\phi$, dynamics model $\hat{T}_\theta$, environment buffer $\mathcal{D}_e$, model buffer $\mathcal{D}_m$
				\REPEAT
				\STATE Take an action in the environment using the policy $\pi_\phi$; add the sample$(s,a,s',r) $ to $\mathcal{D}_e$
				\IF{every $E$ real timesteps are finished}
				\STATE Perform $G_1$ gradient steps to train the model $\hat{T}_\theta$ with samples from $\mathcal{D}_e$
				\FOR{$F$ model rollouts}
				\STATE Sample a state $s$ uniformly from $\mathcal{D}_e$
				\STATE Use policy $\pi_\phi$ to perform a $k$-step model rollout starting from $s$; add to $\mathcal{D}_m$
				\ENDFOR
				\STATE Perform $G_2$ gradient steps to train the feature extractor with samples $(s,a)$ from both $\mathcal{D}_e$ and $\mathcal{D}_m$ by the model adaptation loss $\mathcal{L}_{\text{WD}}$
				\ENDIF
				\STATE Perform $G_3$ gradient steps to train the policy $\pi_\phi$ with samples $(s,a,s',r)$ from $\mathcal{D}_e\cup\mathcal{D}_m$
				\UNTIL{certain number of real samples}
			\end{algorithmic}
\end{algorithm}
\vspace{-10pt}

To optimize the occupancy measure matching term over the model, instead of alleviating the distribution mismatch problem on data level, we tackle it explicitly on feature level from the perspective of unsupervised domain adaptation \cite{ben2010theory, zhao2019learning}, which aims at generalizing a learner on unlabeled data with labeled data from a different distribution. One promising solution for domain adaptation is to find invariant feature representations by incorporating an additional objective of feature distribution alignment \cite{ben2007analysis, ganin2016domain}. Inspired by this, we propose to introduce a model adaptation procedure to encourage the dynamics model to learn the features that are invariant to the real state-action data and the simulated one.

Model adaptation can be seamlessly incorporated into existing Dyna-style MBRL methods since it is orthogonal to them, including those by reducing the distribution mismatch problem. In this paper, we adopt MBPO \cite{mbpo} as our baseline backbone framework due to its remarkable success in practice. We dub the integrated framework AMPO and detail the algorithm in Algorithm~\ref{alg}.

\subsection{Preliminary: MBPO Algorithm}
\paragraph{Model Learning}
We use a bootstrapped ensemble of probabilistic dynamics models $\{\hat{T}_\theta^1, ..., \hat{T}_\theta^B \}$ to capture model uncertainty, which was first introduced in \cite{pets} and has shown to be effective in model learning \cite{mbpo, wang2019exploring}. Here $B$ is the ensemble size and $\theta$ denotes the parameters used in the model ensemble. To be specific, each individual dynamics model $\hat{T}_\theta^i$ is a probabilistic neural network which outputs a Gaussian distribution with diagonal covariance conditioned on the state $s_n$ and the action $a_n$:
$ \hat{T}_\theta^i(s_{n+1}\mid s_n, a_n) = \mathcal{N}(\mu_\theta^i(s_n,a_n),\Sigma_\theta^i(s_n,a_n)). $
The neural network models in the ensemble are initialized differently and trained with different bootstrapped samples selected from the environment buffer $\mathcal{D}_e$, which stores the real data collected from the environment. To train each single model, the negative log-likelihood loss is used: 
\begin{equation}
\label{eq:model-loss}
\mathcal{L}_{\hat{T}}^i(\theta)=~\sum_{n=1}^{N}\left[\mu_{\theta}^i\left({s}_{n}, {a}_{n}\right)-{s}_{n+1}\right]^{\top} {{\Sigma}_{\theta}^i}^{-1}\left({s}_{n}, {a}_{n}\right) \left[\mu_{\theta}^i\left({s}_{n}, {a}_{n}\right)-{s}_{n+1}\right]+\log \operatorname{det} {\Sigma}_{\theta}^i\left({s}_{n}, {a}_{n}\right).
\end{equation}

\paragraph{Model Usage}
The ensemble models are used to generate $k$-length simulated rollouts branched from the states sampled from the environment buffer $\mathcal{D}_e$. In detail, at each step, a model from the ensemble is selected at random to predict the next state and then the simulated data is added to the model buffer $\mathcal{D}_m$. Then a policy is trained on both real and simulated data from two buffers with a certain ratio. We use soft actor-critic (SAC) \cite{sac} as the policy optimization algorithm, which trains a stochastic policy with entropy regularization in actor-critic architecture by minimizing the expected KL-divergence:
\begin{equation}
\mathcal{L}_\pi(\phi) = \mathbb{E}_s[D_{\text{KL}}(\pi_\phi(\cdot|s) ~\|~ \exp(Q(s_t,\cdot)-V(s))].
\end{equation}

\begin{figure}[t]
	\includegraphics[width=\columnwidth]{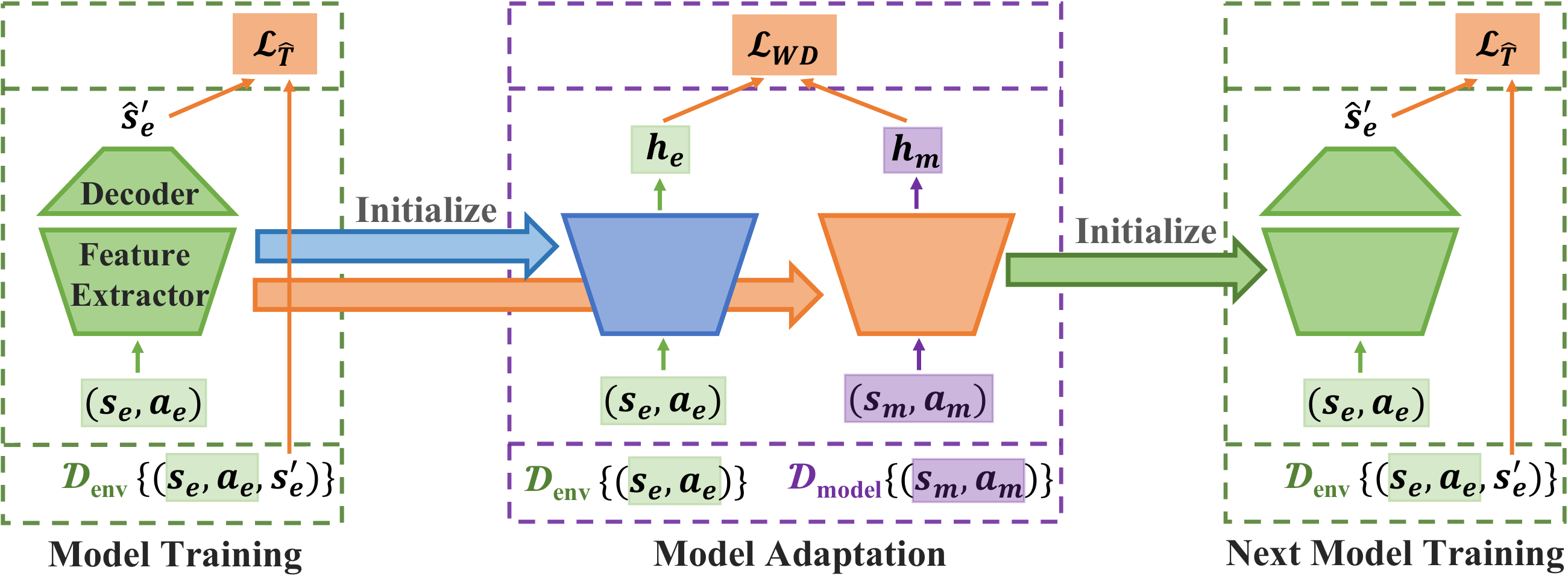}
	\caption{Illustration of model training and model adaptation. At every iteration, the model is learned by maximum likelihood estimation with real data collected from the environment. After the model training, the feature extractor is copied, and then the model adaptation begins where the two separate feature extractors are used for real data and simulated data respectively. After the model adaptation at the current iteration is finished, the feature extractor for the simulated data will be used to initialize the model training at the next iteration.}
	\label{fig:train-adapt}
\end{figure}

\subsection{Incorporating Unsupervised Model Adaptation}
\label{sec:model-adaptation}

For convenience, in the following, we only consider one individual dynamics model, and the same procedure could be applied to any other dynamics model in the ensemble.
Since the model is implemented by a neural network, we define the first several layers as the feature extractor $f_g$ with corresponding parameters $\theta_g$ and the remaining layers as the decoder $f_d$ with parameters $\theta_d$. Thus we have $\hat{T}=f_d\circ f_g$ and $\theta=\{\theta_g, \theta_d \}$. We propose to add a model adaptation loss over the output of feature extractor, which encourages such a conceptual division as the feature encoder and the decoder.
The main idea of model adaptation is to adjust the feature extractor $f_g$ in order to align the two feature distributions of real samples and simulated ones as input, so that the induced feature distributions from real and simulated samples are close in the feature space.

To incorporate unsupervised model adaptation into MBPO, we adopt alternative optimization between model training and model adaptation as illustrated in Figure \ref{fig:train-adapt}.  At every iteration (line 4 to 11 in Algorithm~\ref{alg}), when the dynamics model is trained, we use it to generate simulated rollouts which will then be used for model adaptation and policy optimization. 
As for the detailed adaptation strategy, instead of directly sharing the parameter weights of the feature extractor between real data and simulated data \cite{ganin2016domain}, we choose to adopt the asymmetric feature mapping strategy \cite{adda}, which has been shown to outperform the weight-sharing variant in domain adaptation due to more flexible feature mappings. To be specific, the asymmetric feature mapping strategy unties the shared weights between two domains and learns individual feature extractors for real data and simulated data respectively. Thus in AMPO, after the model adaptation at one iteration is finished, we will use the weight parameters for simulated data to initialize the model training for the next iteration. Through such an alternative optimization between model training and model adaptation, the feature representations learned by the feature extractor will be informative for the decoder to predict real samples, and more importantly it can generalize to the simulated samples.

\subsection{Model Adaptation via Wasserstein-1 Distance}

Specifically, given real samples $(s_e, a_e)$ from the environment buffer $\mathcal{D}_e$ and the simulated samples $(s_m, a_m)$ from the model buffer $\mathcal{D}_m$, the two separate feature extractors map them to feature representations $h_e=f_g^e(s_e,a_e)$ and $h_m=f_g^m(s_m,a_m)$. To achieve model adaptation, we minimize one kind of IPM between the two feature distributions $\mathbb{P}_{h_e}$ and $\mathbb{P}_{h_m}$ according to the lower bound in Theorem \ref{theorem}. In this paper, we choose Wasserstein-1 distance as the divergence measure in model adaptation, which is validated to be effective in domain adaptation \cite{shen2017wasserstein}. In the appendix, we also provide a variant that uses Maximum Mean Discrepancy.

Wasserstein-1 distance corresponds to IPM where the witness function satisfies the 1-Lipschitz constraint.
To estimate the Wasserstein-1 distance, we use a critic network $f_c$ with parameters $\omega$ as introduced in Wasserstein GAN \cite{wgan}. The critic maps a feature representation to a real number, and then according to Eq.~\eqref{eq:ipm} the Wasserstein-1 distance can be estimated by maximizing the following objective function over the critic:
\begin{equation}
	\mathcal{L}_{\text{WD}}(\theta_g^e, \theta_g^m, \omega) = \frac{1}{N_e} \sum_{i=1}^{N_e} f_c(h_e^i) - \frac{1}{N_m} \sum_{j=1}^{N_m} f_c(h_m^j).
\end{equation}
In the meanwhile, the parameterized family of critic functions $\{f_c\}$ should satisfy 1-Lipschitz constraint according to the IPM formulation of Wasserstein-1 distance. In order to properly enforce 1-Lipschitz, we choose the gradient penalty loss \cite{wgan-gp} for the critic
\begin{eqnarray}
\mathcal{L}_{gp} (\omega) = \mathbb{E}_{\mathbb{P}_{\hat{h}}}[(\lVert \nabla f_c(\hat{h}) \rVert_2 - 1)^2 ] ,
\end{eqnarray}
where $\mathbb{P}_{\hat{h}}$ is the distribution of uniformly distributed linear interpolations of $\mathbb{P}_{h_e}$ and $\mathbb{P}_{h_m}$. 

After the critic is trained to approximate the Wasserstein-1 distance, we optimize the feature extractor to minimize the estimated Wasserstein-1 distance to learn features invariant to the real data and simulated data. To sum up, model adaptation though Wasserstein-1 distance can be achieved by solving the following minimax objective
\begin{equation}
\label{obj:wd-loss}
\min_{\theta_g^e, \theta_g^m}\max_{\omega}\quad \mathcal{L}_{\text{WD}}(\theta_g^e, \theta_g^m, \omega) - \alpha\cdot\mathcal{L}_{gp}(\omega),
\end{equation}
where $\theta_g^e$ and $\theta_g^m$ are the parameters of the two feature generators for real data and simulated data respectively, and $\alpha$ is the balancing coefficient. For model adaptation at each iteration, we alternate between training the critic to estimate the Wasserstein-1 distance and training the feature extractor of the dynamics model to learn transferable features.

\begin{figure*}[!t]
	\centering
	\includegraphics[width=\textwidth]{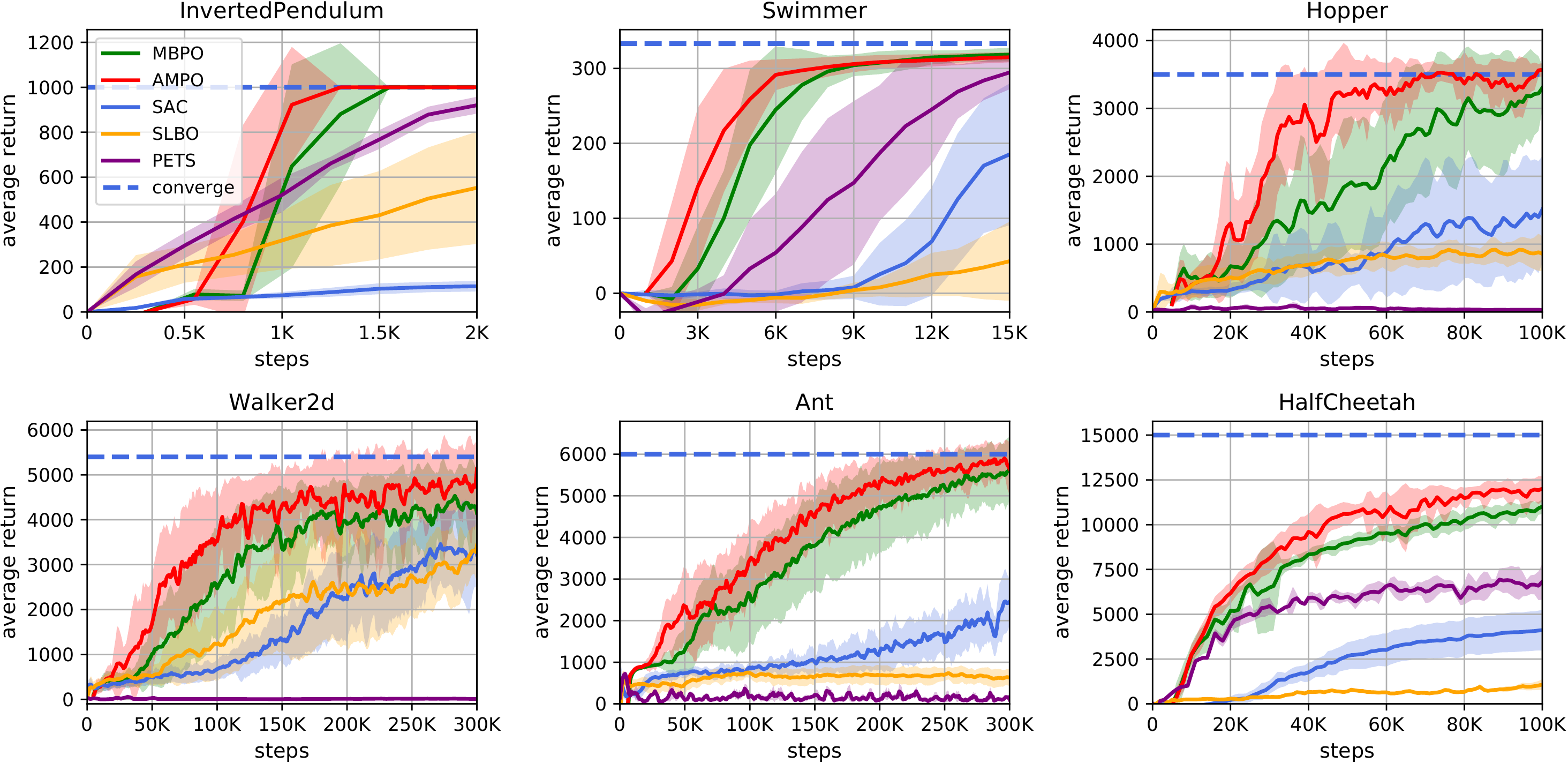}
	\caption{Performance curves of AMPO and other model-based and model-free baselines on six continuous control benchmarking environments. We average the results over five random seeds, where solid curves depict the mean of five trials and shaded areas indicate the standard deviation. The dashed reference lines are the asymptotic performance of Soft Actor-Critic (SAC).} 
	\label{fig:result}
\end{figure*}

\section{Experiments}
\subsection{Comparative Evaluation} 
\textbf{Compared Methods}~~
We compare our method AMPO to other model-free and model-based algorithms. Soft Actor-Critic (SAC) \cite{sac} is the state-of-the-art model-free off-policy algorithm in terms of sample efficiency and asymptotic performance so we choose SAC for the model-free baseline. For model-based methods, we compare to MBPO \cite{mbpo}, PETS \cite{pets} and SLBO \cite{slbo}.

\textbf{Environments}~~
We evaluate AMPO and other baselines on six MuJoCo continuous control tasks with a maximum horizon of 1000 from OpenAI Gym~\cite{brockman2016openai}, including InvertedPendulum, Swimmer, Hopper, Walker2d, Ant and HalfCheetah. For the Swimmer environment, we use the modified version introduced by \cite{mbbl} since the original version is quite difficult to solve. For the other five environments, we adopt the same settings as in \cite{mbpo}.

\textbf{Implementation Details}~~
We implement all our experiments using TensorFlow.\footnote{Our code is publicly available at: \url{https://github.com/RockySJ/ampo}} 
For MBPO and AMPO, we first apply a random policy to sample a certain number of real data and use them to pre-train the dynamics model. In AMPO, the model adaptation procedure will not be executed any more after a certain number of real samples, which doesn't affect performance.
In each adaptation iteration, we train the critic for five steps and then train the feature extractor for one step, and the coefficient $\alpha$ of gradient penalty is set to 10. Every time we train the dynamics model, we randomly sample several real data as a validation set and stop the model training if the model loss does not decrease for five gradient steps, which means we do not choose a specific value for the hyperparameter $G_1$.
Other important hyperparameters are chosen by grid search and detailed hyperparameter settings used in AMPO can be found in the appendix.

\textbf{Results}~~
The learning curves of all compared methods are presented in Figure \ref{fig:result}. From the comparison, we observe that our approach AMPO is the most sample efficient as they learn faster than all other baselines in all six environments. Furthermore, AMPO is capable of reaching comparable asymptotic performance of the state-of-the-art model-free baseline SAC. Compared with MBPO, our approach achieves better performance in all the environments, which verifies the value of model adaptation. This also indicates that even in the situation with reduced distribution mismatch by using short rollouts, model adaptation still helps.

\subsection{Model Errors}
\label{sec:model-loss}

To better understand how model adaptation affects model learning, we plot in Figure \ref{fig:model-loss} the curves of one-step model losses in two environments. By comparison, we observe that both the training and validation losses of dynamics models in AMPO are smaller than that in MBPO throughout the learning process. It shows that by incorporating model adaptation the learned model becomes more accurate. Consequently, the policy optimized based on the improved dynamics model can perform better.

We also investigate the compounding model errors of multi-step forward predictions, which is largely caused by the distribution mismatch problem. The $h$-step compounding error \cite{nn} is calculated as $\epsilon_h=\frac{1}{h}\sum_{i=1}^h \left\| \hat{s}_i-s_i\right\|^2$ where $\hat{s}_{i+1}=\hat{T}_\theta(\hat{s}_i, a_i)$ and $\hat{s}_0 = s_0$. From Figure~\ref{fig:compounding-error} we observe that AMPO achieves smaller compounding errors than MBPO, which verifies that AMPO can successfully mitigate the distribution mismatch.

\begin{figure*}[t!]
	\centering
	\subfigure[One-step model losses.]{
		\includegraphics[width=0.652\textwidth]{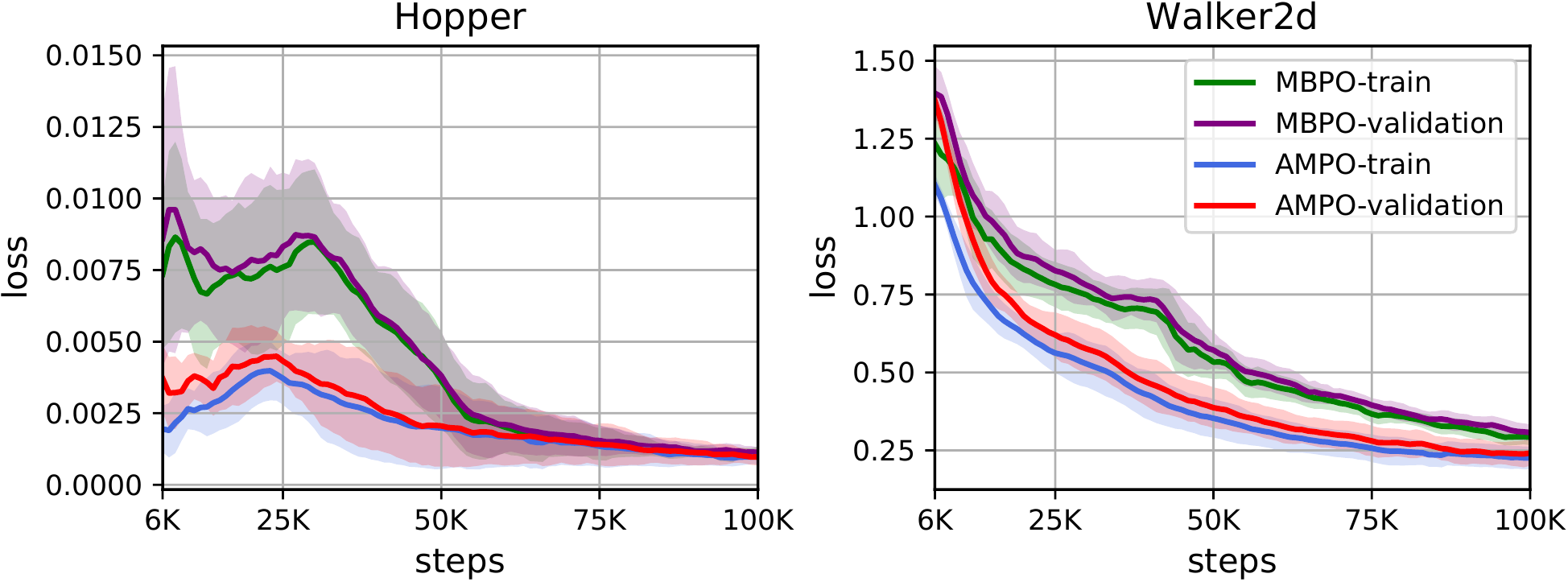}
		\label{fig:model-loss}
	}
	\subfigure[Compounding errors.]{
		\includegraphics[width=0.308\textwidth]{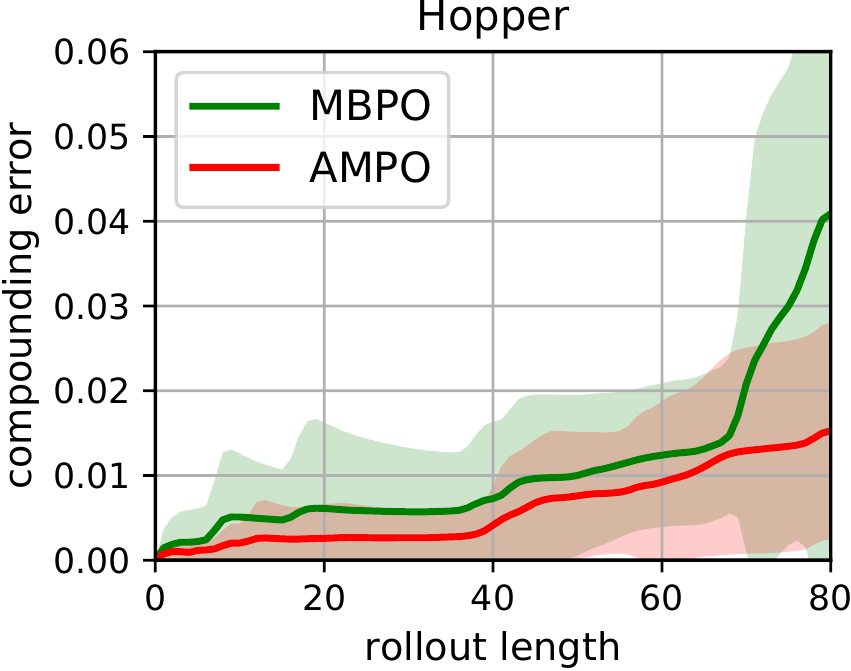}
		\label{fig:compounding-error}
	}
	\caption{(a) The one-step model losses are evaluated on the training and (varying) validation data set from the environment buffer every time the model is trained. (b) Every 5000 environment steps in Hopper, we calculate multi-step compounding errors and then average them.}
\end{figure*}

\subsection{Wasserstein-1 Distance Visualization}

To further investigate the effect of model adaptation, we visualize the estimated Wasserstein-1 distance between the real features and simulated ones. Besides MBPO and AMPO, we additionally analyze the multi-step training loss of SLBO since it also uses the model output as the input of model training, which may help learn invariant features. According to the results shown in Figure~\ref{fig:wasserstein}, we find that: i) the vanilla model training in MBPO itself can slowly minimize the Wasserstein-1 distance between feature distributions; ii) the multi-step training loss in SLBO does help learn invariant features but the improvement is limited; iii) the model adaptation loss in AMPO is effective in promoting feature distribution alignment, which is consistent with our initial motivation.

\subsection{Hyperparameter Studies}

In this section, we study the sensitivity of AMPO to important hyperparameters, and the results in Hopper are shown in Figure~\ref{fig:hyperparameter}. We first conduct experiments with different adaptation iterations $G_2$. We observe that increasing $G_2$ yields better performance up to a certain level while too large $G_2$ degrades the performance, which means that we need to control the trade-off between model training and model adaptation to ensure the representations to be invariant and also discriminative. 
We then conduct experiments with different rollout length schedules, of which the effectiveness has been shown in MBPO \cite{mbpo}. We observe that generating longer rollouts earlier in AMPO improves the performance while it degrades the performance of MBPO a little. It is easy to understand since as discussed in Section~\ref{sec:model-loss} the learned dynamics model in AMPO obtains better accuracy in approximations and therefore longer rollouts can be performed. 

\section{Related Work}
The two important issues in MBRL methods are model learning and model usage. Model learning mainly involves two aspects: (1) function approximator choice like Gaussian process \cite{deisenroth2011pilco}, time-varying linear models \cite{levine2016end} and neural networks \cite{nn}, and (2) objective design like multi-step L2-norm \cite{slbo}, log loss \cite{pets} and adversarial loss \cite{mi}. Model usage can be roughly categorized into four groups: (1) improving policies using model-free algorithms like Dyna \cite{dyna, slbo, clavera2018model, mbpo}, (2) using model rollouts to improve target value estimates for temporal difference (TD) learning \cite{feinberg2018model, steve}, (3) searching policies with back-propagation through time by exploiting the model derivatives \cite{deisenroth2011pilco, levine2016end}, and (4) planning by model predictive control (MPC) \cite{nn, pets} without explicit policy. The proposed AMPO framework with model adaptation can be viewed as an innovation in model learning by additionally adopting an adaptation loss function. 
\begin{figure*}[t!]
	\centering
	\subfigure[Wasserstein distance.]{
		\includegraphics[width=0.311\textwidth]{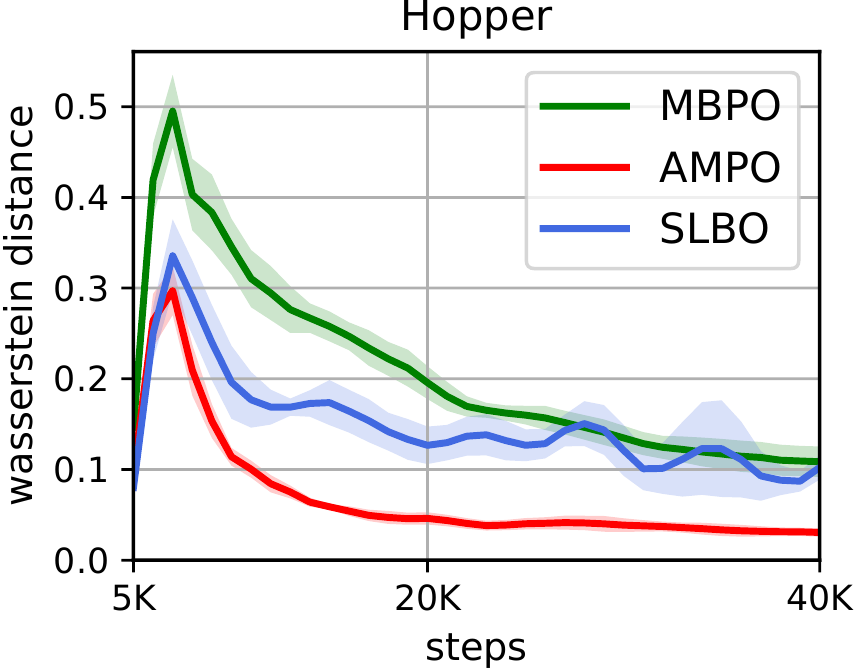}
		\label{fig:wasserstein}
	}
	\subfigure[Hyperparameter studies.]{
		\includegraphics[width=0.649\textwidth]{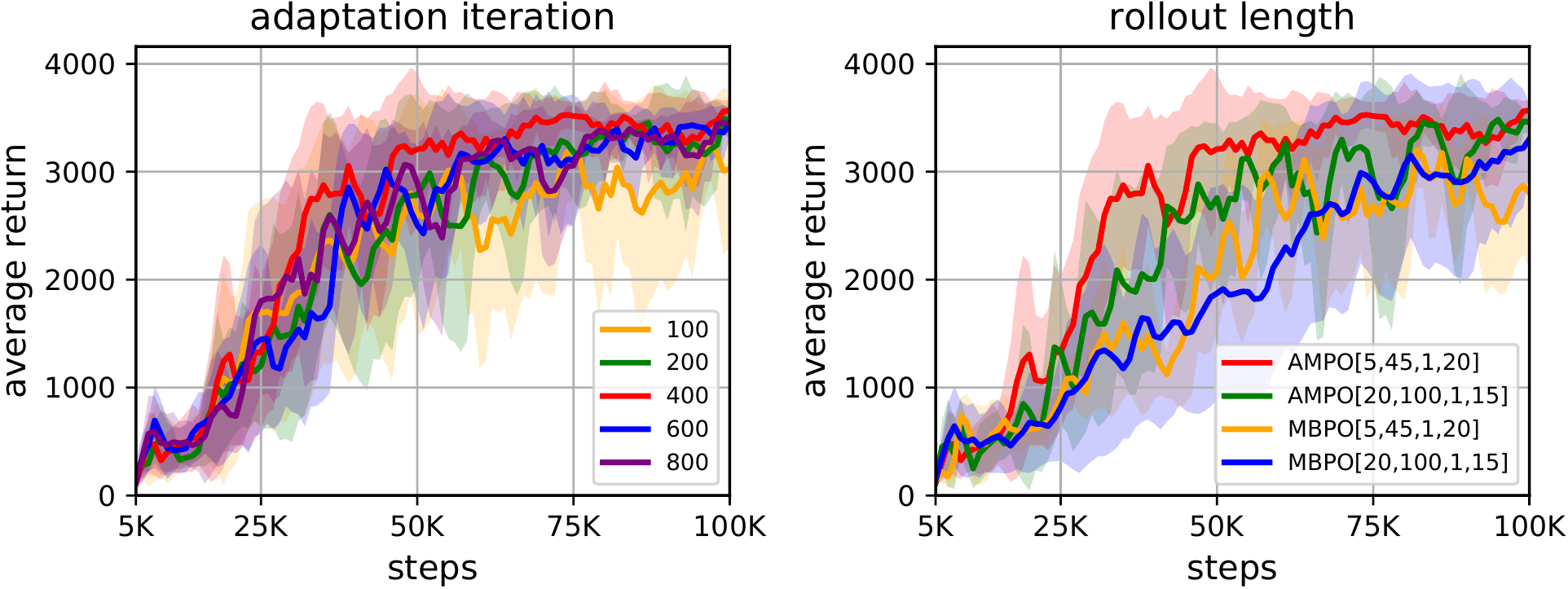}
		\label{fig:hyperparameter}
	}
	\caption{(a) We visualize the Wasserstein-1 distance between the feature distributions. (b) We study the effect of the number of adaptation iterations and the rollout length in AMPO. $[a,b,x,y]$ means the rollout length linearly increases from $x$ to $y$ at the epochs between $a$ and $b$. The rollout schedule [20,100,1,15] is the value used in MBPO and [5,45,1,20] is the schedule we choose for AMPO.}
\end{figure*}

In this paper, we mainly focus on the distribution mismatch problem in deep MBRL \cite{talvitie2014model}, \textit{i.e.}, the state-action occupancy measure used for model learning mismatches the one generated for model usage. 
Several previous methods have been proposed to reduce the distribution mismatch problem. Firstly, it can be reduced by improving model learning, such as using probabilistic model ensembles \cite{pets}, designing multi-step models \cite{asadi2019combating} and adopting a generative adversarial imitation objective \cite{mi}. Secondly, it can be reduced by designing delicate schemes in model usage, such as using short model-generated rollouts \cite{mbpo} and interpolating between rollouts of various lengths \cite{steve}. Although these existing methods help alleviate the distribution mismatch, it has not been solved explicitly. On the other hand, the multi-step training loss in SLBO \cite{slbo} and the \textit{self-correct} mechanism \cite{talvitie2014model, self-correct} can solve this problem. They may also help learn invariant features since the model predicted states were used as the input to train the model in addition to the real data. By comparison, model adaptation directly enforces the distribution alignment constraint to mitigate the problem, and the decoder in AMPO is only trained by real data to guarantee it is unbiased.

Previous theoretical works on MBRL mostly focused on either the tabular MDPs or linear dynamics \cite{szita2010model, jaksch2010near, dean2017sample, simchowitz2018learning}, but not much in the context of continuous state space and non-linear systems. Recently, \cite{slbo} gave theoretical guarantee of monotonic improvement by introducing a reference policy and imposing constraints on policy optimization and model learning related to the reference policy. Then \cite{mbpo} also derived a lower bound focusing on branched short rollouts and the algorithm was designed intuitively instead of maximizing the lower bound.

\section{Conclusion}

In this paper, we investigate how to explicitly tackle the distribution mismatch problem in MBRL. We first provide a lower bound to justify the necessity of model adaptation to correct the potential distribution bias in MBRL. We then propose to incorporate unsupervised model adaptation with the intention of aligning the latent feature distributions of real data and simulated data. In this way, the model gives more accurate predictions when generating simulated data, and therefore the follow-up policy optimization performance can be improved. Extensive experiments on continuous control tasks have shown the effectiveness of our work. As a future direction, we plan to integrate additional domain adaptation techniques to further promote distribution alignment. We believe our work takes an important step towards more sample-efficient MBRL. 

\section*{Broader Impact}

The proposed model adaptation can be incorporated into existing Dyna-style model-based methods, such as MBPO in this paper, to further improve the sample efficiency. This improvement will ease the application of MBRL in practical decision-making problems like robotic control in the future. Despite the potential positive impacts of model adaptation, we should also notice some negative issues. It will cost more to tune real-world MBRL systems with model adaptation to avoid too strong or too weak adaptation, which is usually related to specific environments. We hope our work can provide insights for future improvements in tackling the distribution mismatch problem in MBRL.

\section*{Acknowledgments}
The corresponding author Weinan Zhang is supported by "New Generation of AI 2030" Major Project (2018AAA0100900) and National Natural Science Foundation of China (61702327, 61772333, 61632017, 81771937).

\bibliography{reference}
\bibliographystyle{apalike}

\newpage
\appendix

\onecolumn

\section*{\centering Appendix for: Model-based Policy Optimization with \\ Unsupervised Model Adaptation}
\section{Omitted Proofs}
\decomposition*

\begin{proof}
For the state visit distribution $\nu^\pi_{\hat{T}}(s)$, we have 

\begin{equation}
    \nu^\pi_{\hat{T}}(s') = (1-\gamma) \nu_0(s') + \gamma \int \rho^\pi_{\hat{T}}(s,a) {\hat{T}}(s'|s,a) \dif s \dif a 
\end{equation}
where $\nu_0$ denotes the probability of the initial state being the state $s'$. Then we have 
\begin{equation}
	\begin{aligned}
	\label{eq: sd}
	& \quad\,\, |\nu^{\pi_D}_{T}(s')  - \nu^\pi_{\hat{T}}(s')| \\
	&= \gamma \left|\int_{s,a} T(s'|s,a) \rho_{T}^{\pi_D}(s,a) - \hat{T}(s'|s,a) \rho_{\hat{T}}^{\pi}(s,a) \dif s \dif a \right| \\
	&= \gamma \left| \mathbb{E}_{(s,a) \sim \rho_{T}^{\pi_D}}[T(s'|s,a)] - \mathbb{E}_{(s,a) \sim \rho_{\hat{T}}^\pi}[\hat{T}(s'|s,a)]\right| \\
	&\leq \gamma \left|\mathbb{E}_{(s,a) \sim \rho_{T}^{\pi_D}} [T(s'|s,a) -\hat{T}(s'|s,a)]\right| + \gamma \left| \mathbb{E}_{(s,a) \sim \rho_{T}^{\pi_D}}[\hat{T}(s'|s,a)] - \mathbb{E}_{(s,a) \sim \rho_{\hat{T}}^{\pi}}[\hat{T}(s'|s,a)]\right|\\
	&\leq \gamma \mathbb{E}_{(s,a) \sim \rho_{T}^{\pi_D}} \left|T(s'|s,a) -\hat{T}(s'|s,a)\right| + \gamma d_{\mathcal{F}_{s'}}(\rho_{T}^{\pi_D}, \rho_{\hat{T}}^{\pi}),
	\end{aligned}
	\end{equation}
which completes the proof.
\end{proof}

\main*
\begin{proof}
The return discrepancy is bounded as follows
	\begin{equation}
	\begin{aligned}
	\label{eq: bound}
	\quad\,\, |\eta(\pi) - \hat{\eta}(\pi)|  & = \left| \int_{s,a} \left(\rho^\pi_{T}(s,a) - \rho^\pi_{\hat{T}}(s,a) \right) r(s,a)  \dif s  \dif a \right| \\
	& = \left| \int_{s,a} \left(\nu^\pi_{T}(s)  \pi(a|s) - \nu^\pi_{\hat{T}}(s)  \pi(a|s) \right) r(s,a)  \dif s  \dif a \right| \\
	& \leq R \cdot  \int_{s,a} \left|\nu^\pi_{T}(s)  \pi(a|s) - \nu^\pi_{\hat{T}}(s)  \pi(a|s) \right|   \dif s  \dif a  \\
	& = R \cdot \int_{s} \left| \nu^\pi_{T}(s)  - \nu^\pi_{\hat{T}}(s) \right| \dif s  \\
	& = R \cdot \int_{s} \left|\nu^{\pi_D}_{T}(s)  - \nu^\pi_{\hat{T}}(s) + \nu^{\pi}_{T}(s)  - \nu^{\pi_D}_{T}(s) \right|  \dif s  \\
	& \leq R \cdot \int_{s} \left| \nu^{\pi_D}_{T}(s)  - \nu^\pi_{\hat{T}}(s) \right|  \dif s  + R \cdot \epsilon_\pi \\
	\end{aligned}
	\end{equation}
Replacing the above state $s$ with the notation $s'$, then according to Lemma~\ref{lemma}, we have
\begin{equation}
\begin{aligned}
&\quad\,\, |\eta(\pi) - \hat{\eta}(\pi)| \\ & \leq R \cdot \epsilon_\pi + \gamma R \cdot \mathbb{E}_{(s,a) \sim \rho_{T}^{\pi_D}} \int_{s'} \left| T(s'|s,a) -\hat{T}(s'|s,a) \right| \dif s' + \gamma R \cdot \int_{s'} d_{\mathcal{F}_{s'}}(\nu_{T}^{\pi_D}, \nu_{\hat{T}}^{\pi})\dif s'  \\
& \leq R \cdot \epsilon_\pi + \gamma R \cdot \mathbb{E}_{(s,a) \sim \rho_{T}^{\pi_D}} \int_{s'} \left| T(s'|s,a) -\hat{T}(s'|s,a) \right| \dif s' + \gamma R \cdot d_{\FF}(\rho_{T}^{\pi_D}, \rho_{\hat{T}}^{\pi}) \cdot \text{Vol}(\sspace) \\
& = R \cdot \epsilon_\pi + 2 \gamma R \cdot \mathbb{E}_{(s,a) \sim \rho_{T}^{\pi_D}} d_\text{TV}(T(\cdot|s,a), \hat{T}(\cdot|s,a)) + \gamma R \cdot d_{\FF}(\rho_{T}^{\pi_D}, \rho_{\hat{T}}^{\pi}) \cdot \text{Vol}(\sspace) \\
& \leq R \cdot \epsilon_\pi + \gamma R \cdot \mathbb{E}_{(s,a) \sim \rho_{T}^{\pi_D}} \sqrt{2 D_\text{KL}(T(\cdot|s,a), \hat{T}(\cdot|s,a))} + \gamma R \cdot d_{\FF}(\rho_{T}^{\pi_D}, \rho_{\hat{T}}^{\pi}) \cdot \text{Vol}(\sspace)~,\\
\end{aligned}
\end{equation}
where the last inequality holds due to Pinsker’s inequality, which completes the proof.
\end{proof}

\section{Hyperparameters Settings}

\begin{table}[htbp]
	\small
	\centering
	\caption{Hyperparameter settings for AMPO results. $[a,b,x,y]$ denotes a thresholded linear function, \emph{i.e.} at epoch $e$, $f(e) = \min(\max(x+\frac{e-a}{b-a} \cdot (x-y), x), y)$.}
	\begin{tabular}{cccccccc}\toprule
	\label{table:hyper}
		&  & \rotatebox[origin=c]{60}{\parbox{1cm}{\textbf{Inverted\\Pendulum}}} & \multicolumn{1}{c}{\rotatebox[origin=c]{60}{\textbf{Swimmer}}} & \multicolumn{1}{c}{\rotatebox[origin=c]{60}{\textbf{Hopper}}} & \multicolumn{1}{c}{\rotatebox[origin=c]{60}{\textbf{Walker2d}}} & \multicolumn{1}{c}{\rotatebox[origin=c]{60}{\textbf{Ant}}} &\rotatebox[origin=c]{60}{\parbox{1cm}{~~~\textbf{Half\\Cheetah}}} \\
		\midrule
		& \multirow{2}{*}{network architecture} & \multicolumn{6}{c}{{MLP with four hidden layers of size 200}} \\
		& & \multicolumn{6}{c}{feature extractor: four hidden layers; decoder: one output layer} \\
		\hline
		& real samples for & \multirow{2}{*}{300}& \multicolumn{1}{|c}{\multirow{2}{*}{2000}}&\multicolumn{3}{|c}{\multirow{2}{*}{5000}}\\
		& model pretraining &  & \multicolumn{1}{|c}{}& \multicolumn{3}{|c}{}\\
		\hline
		& real steps & \multirow{2}{*}{250} & \multicolumn{5}{|c}{\multirow{2}{*}{1000}}\\
		& per epoch &  & \multicolumn{5}{|c}{} \\
		\hline
		& model adaptation & \multirow{2}{*}{64} & \multicolumn{5}{|c}{\multirow{2}{*}{256}}\\
		&  batch size &  & \multicolumn{5}{|c}{} \\
		\hline
		\multirow{2}{*}{$E$} & real steps between & \multirow{2}{*}{125} & \multicolumn{5}{|c}{\multirow{2}{*}{250}} \\
		& model training & & \multicolumn{5}{|c}{} \\
		\hline
		\multirow{2}{*}{$F$} & model rollout & \multicolumn{6}{c}{\multirow{2}{*}{100000}} \\
		& batch size & \multicolumn{6}{c}{} \\
		\hline
		\multirow{2}{*}{$B$} & \multirow{2}{*}{ensemble size} & \multicolumn{6}{c}{\multirow{2}{*}{7}} \\
		& & \multicolumn{6}{c}{} \\
		\hline
		\multirow{2}{*}{$G_3$} & policy updates & \multirow{2}{*}{30}&  \multicolumn{4}{|c|}{\multirow{2}{*}{20}} & \multirow{2}{*}{40} \\
		& per real step   &  &\multicolumn{4}{|c|}{}& \\
		\hline
		\multirow{2}{*}{$k$} & \multirow{2}{*}{rollout length} & \multirow{2}{*}{1} & \multirow{2}{*}{1} & \multirow{2}{*}{[5,45,1,20]} & \multirow{2}{*}{1} & \multirow{2}{*}{[10,50,1,20]} &  \multirow{2}{*}{[1,30,1,5]} \\
		& & &  &  &  &   \\
		\hline
		\multirow{2}{*}{$G_2$} & model adaptation & \multirow{2}{*}{6}& \multirow{2}{*}{40} & \multirow{2}{*}{400} & \multirow{2}{*}{1000} & \multirow{2}{*}{3000} & \multirow{2}{*}{[1,30,100,1000]}\\
		&  updates & &  & & & & \\
		\hline
		& model adaptation & \multirow{2}{*}{6} &\multirow{2}{*}{6} & \multirow{2}{*}{40} & \multirow{2}{*}{80} & \multirow{2}{*}{60} & \multirow{2}{*}{30} \\
		& early stop epoch & & &  &  & &  \\
		\bottomrule
	\end{tabular}
\end{table}

\section{MMD Variant of AMPO}
Besides Wasserstein distance, we can use other distribution divergence metrics to align the features. 
MMD is another instance of IPM when the witness function class is the unit ball in a reproducing kernel Hilbert space (RKHS). 
Let $k$ be the kernel of the RKHS $\mathcal{H}_k$ of functions on $\mathcal{X}$. Then the squared MMD in $\mathcal{H}_k$ between two feature distributions $\mathbb{P}_{h_e}$ and $\mathbb{P}_{h_m}$ is \cite{mmd}:
\begin{equation}
\text{MMD}^2_{k}(\mathbb{P}_{h_e}, \mathbb{P}_{h_m})  \defeq ~ \mathbb{E}_{h_e,h'_e}[k(h_e,h'_e)] +\mathbb{E}_{h_m,h'_m}[k(h_m,h'_m)] -2\mathbb{E}_{h_e,h_m}[k(h_e,h_m)],
\end{equation}
which is a non-parametric measurement based on kernel mappings. In practice, given finite feature samples from distributions $ \{ h_e^1,\cdots,h_e^{N_e} \} \sim \mathbb{P}_{h_e} \ $ and $ \{ h_m^1,\cdots,h_m^{N_m} \} \sim \mathbb{P}_{h_m}$, where $N_e$ and $ N_m$ are the number of real samples and simulated ones, one unbiased estimator of $\text{MMD}^2_{k}(\mathbb{P}_{h_e}, \mathbb{P}_{h_m})$ can be written as follows: 
\begin{equation}
\label{obj:mmd-loss}
\mathcal{L}_{\text{MMD}}(\theta_g) = \frac{1}{N_e (N_e-1)} \sum_{i \neq i'} k(h_e^i,h_e^{i'})  +  \frac{1}{N_m (N_m-1)} \sum_{j \neq j'} k(h_m^j,h_m^{j'}) -\frac{2}{N_e N_m} \sum_{i=1}^{N_e} \sum_{j=1}^{N_m} k(h_e^i, h_m^j).
\end{equation} 
To achieve model adaptation through MMD, we optimize the feature extractor to minimize the above adaptation loss $\mathcal{L}_{\text{MMD}}$ with real $(s,a)$ data and simulated one as input.

When implementing the MMD variant, choosing optimal kernels remains an open problem and we use a linear combination of eight RBF kernels with bandwidths $\{0.001, 0.005, 0.01, 0.05, 0.1, 1, 5, 10\}$. The results on three environments are shown in Figure~\ref{fig:result-mmd}. We observe that using MMD as the distribution divergence measure is also effective in the AMPO framework.

\begin{figure*}[!t]
	\centering
	\includegraphics[width=\textwidth]{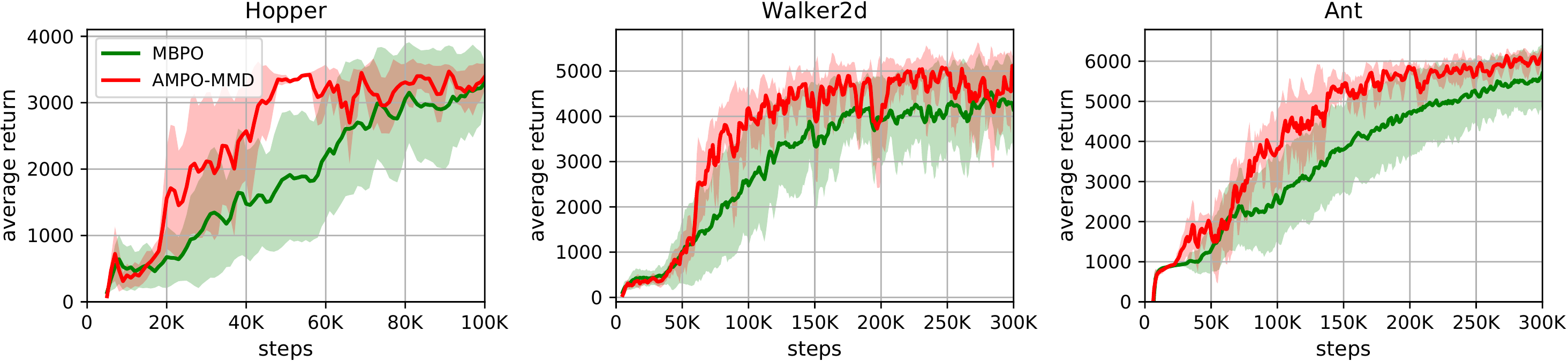}
	\caption{Performance curves of MBPO and MMD variant of AMPO.}
	\label{fig:result-mmd}
\end{figure*}

\section{More Experiment Results}
\subsection{One-step Model Losses}
We show the one-step model losses during the experiments in the other four environments in Figure~\ref{fig:appendix-model-loss}. We find that the conclusion in Section~\ref{sec:model-loss} still holds in these four environments. In InvertedPendulum and Swimmer, the standard deviation is a little larger since the number of real samples for pre-training the model is less.

\subsection{Hyperparameter: Policy Updates}

In MBRL, since we can generate simulated data using the dynamics model, we can take more gradient updates of policy optimization with the simulated data per environment step to accelerate policy learning. However, too many gradient updates for the policy may cause the current model to be inaccurate for the updated policy. Thus the number of policy gradient updates is a quite important hyperparameter in MBRL. We conduct environments with different policy updates, and show the results in Figure~\ref{fig:appendix-policy-update}. We find that when we increase the number of policy updates, the performance of MBPO decreases a little while it doesn't influence AMPO much. It demonstrates that the robustness of AMPO to this hyperparameter.

\subsection{Model Adaptation Early Stopping}

According to the model losses in Figure~\ref{fig:model-loss} and Figure~\ref{fig:appendix-model-loss}, we find that after certain number of environment steps, the model loss difference between AMPO and MBPO becomes small. So in AMPO we early stop the model adaptation procedure after collecting a certain number of real data, such as 40K in the Hopper environment. We then conduct experiments without early stopping model adaptation and the results are demonstrated in Figure~\ref{fig:appendix-nosteop}. We find that keeping adapting the dynamics model throughout the whole learning process does not bring performance improvement. This indicates that model adaptation makes a difference only when the model training data is insufficient. So we set a model adaptation early stopping epoch for each environment (see Table~\ref{table:hyper} for detail) to improve the computation efficiency. 

\subsection{Computation Time}

Since AMPO adds the model adaptation procedure based on MBPO, we would like to see its computation time compared with MBPO. We show in Figure~\ref{fig:appendix-computation-time} the computation time ratio of AMPO against MBPO using the same device. We find that in most environments AMPO needs slightly more computation time than MBPO, and the extra overhead is not much. In InvertedPendulum, however, the computational overhead of AMPO is less than MBPO, of which the reason may be that the next model training in AMPO needs less computation after one model adaptation.

\subsection{Adaptation Strategy}

In AMPO, we untie the feature extractor weights for two data distributions and learn the two feature extractors simultaneously, which is a variant of the adaptation strategy in Adversarial Discriminative Domain Adaptation (ADDA) \cite{adda}. Differently in ADDA the feature mapping for source domain (\ie real data) is fixed. Another alternative is to share the feature extractor weights between the two data distributions. From the comparison in Figure~\ref{fig:appendix-adaptation-strategy}, we observe that the performance of these three adaptation strategies differs not much but AMPO performs slightly better. 
\begin{figure*}[t]
	\centering
	\label{fig:appendix-model-loss}
	\includegraphics[width=\textwidth]{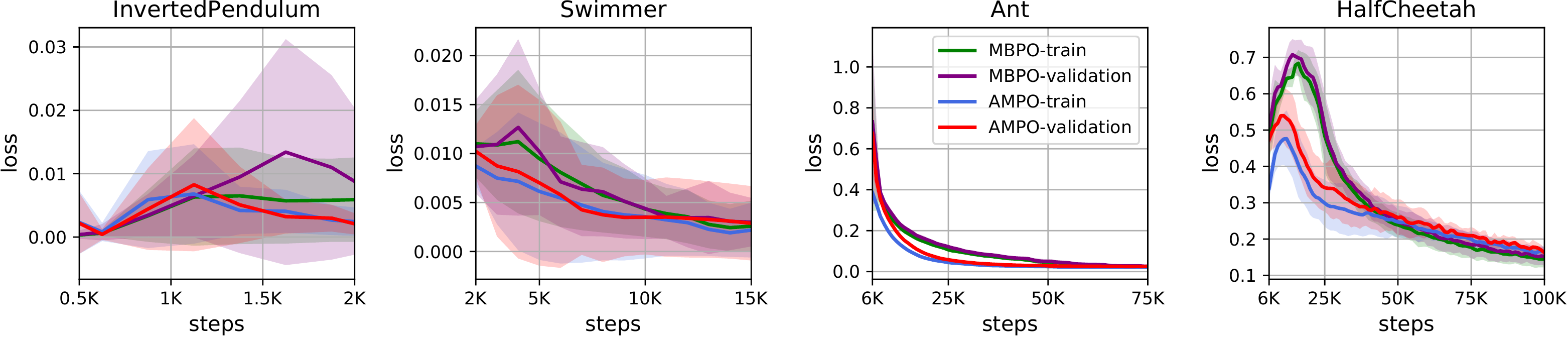}
	\caption{One-step model losses in other four environments.}
\end{figure*}
\begin{figure*}[t]
	\centering
	\subfigure[Policy updates.]{
		\label{fig:appendix-policy-update}
		\includegraphics[width=0.23\textwidth]{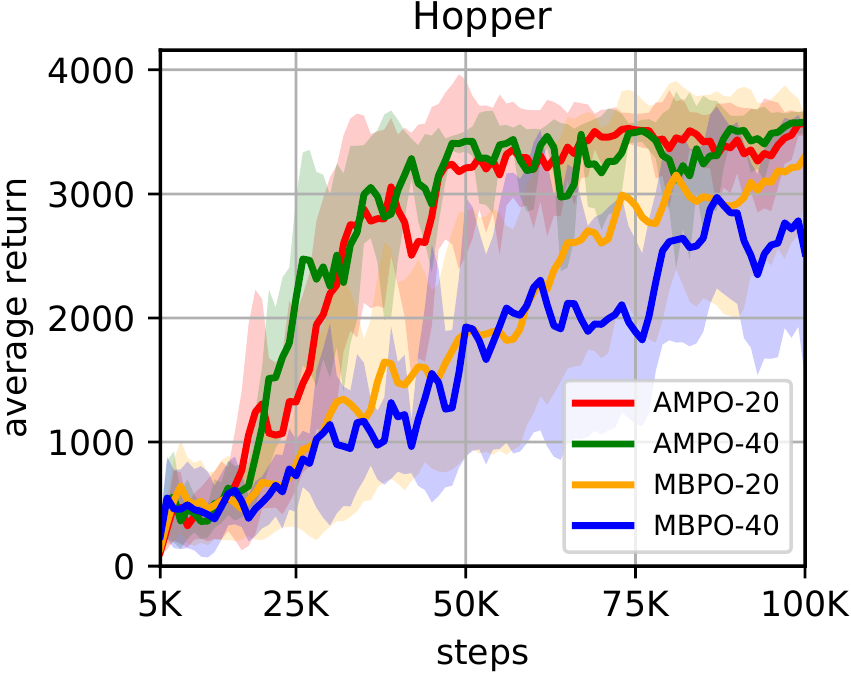}
	}
	\subfigure[Early stopping.]{
		\label{fig:appendix-nosteop}
		\includegraphics[width=0.23\textwidth]{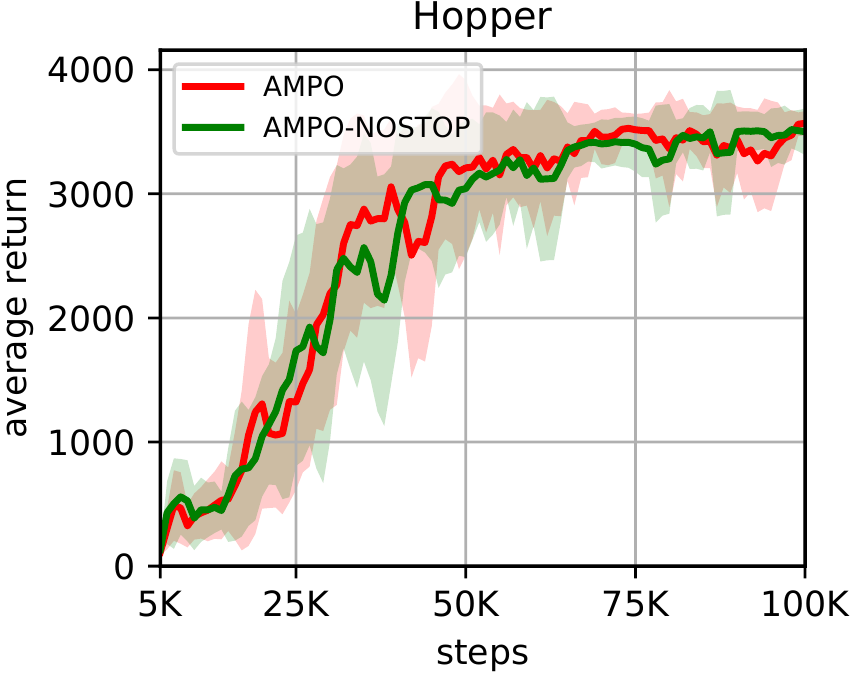}
	}
	\subfigure[Computation time.]{
		\label{fig:appendix-computation-time}
		\includegraphics[width=0.23\textwidth]{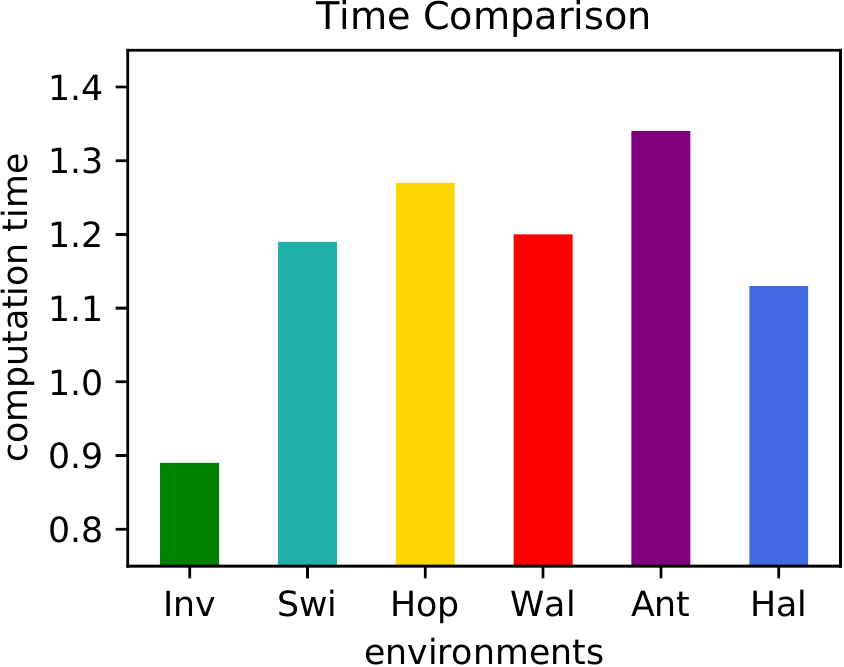}
	}
	\subfigure[Adaptation strategy.]{
		\label{fig:appendix-adaptation-strategy}
		\includegraphics[width=0.23\textwidth]{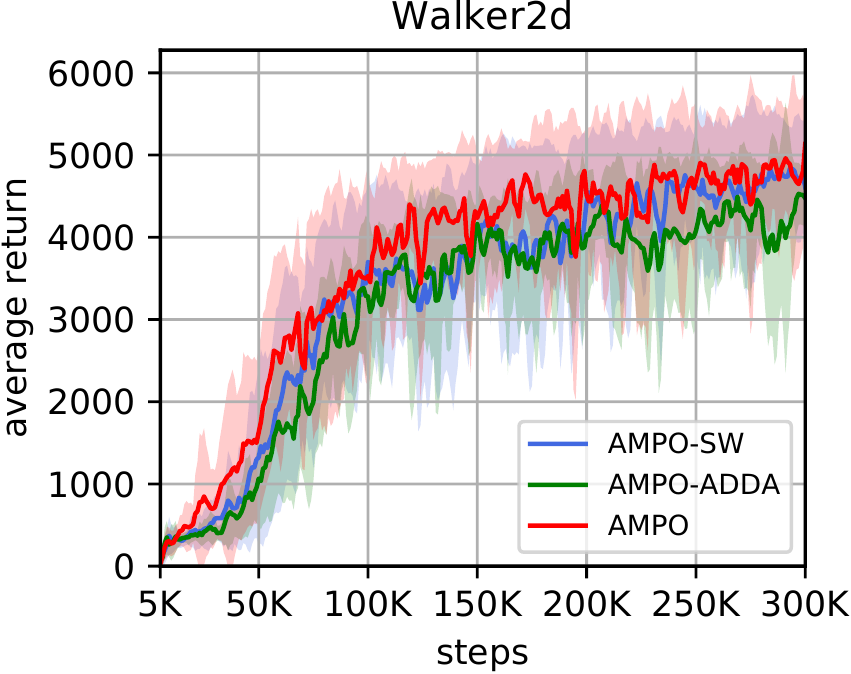}
	}
	\caption{More empirical analysis. (a) AMPO-20 means we use AMPO and the number of policy updates is 20. (b) AMPO-NOSTOP denotes the AMPO variant without early stopping the model adaptation procedure. (c) The ratio of computation time of AMPO to that of MBPO. (d) AMPO-SW denotes the AMPO variant of sharing the feature extractor weights for two data distributions. AMPO-ADDA denotes the AMPO variant of fixing the feature extractor of real data.}
	\label{fig:design}
\end{figure*}

\section{A Different View of Analysis}

In this section, we provide an alternative perspective on the expected return lower bound derivation.

\begin{lemma}
\label{lemma:alternative}
(\cite{slbo}, Lemma 4.3; \cite{mopo}, Lemma 4.1) Let $T$ be the real dynamics and $\hat{T}$ be the dynamics model. Let $G_{\hat{T}}^\pi(s, a) \defeq \mathbb{E}_{s' \sim \hat{T}(\cdot|s,a)}[V^\pi_{T}(s')]- \mathbb{E}_{s' \sim T(\cdot|s,a)}[V^\pi_{T}(s')]$. Then, 
$$
\hat{\eta}[\pi] - \eta[\pi] = \gamma \mathbb{E}_{(s,a)\sim \rho^\pi_{\hat{T}}}[G_{\hat{T}}^\pi(s, a)].
$$
\end{lemma}

Let $\FF_1$ be a collection of functions from $\mathcal{S}\times\mathcal{A}$ to $\RR$ and $\FF_2$ be a collection of functions from $\mathcal{S}$ to $\RR$. With Lemma~\ref{lemma:alternative}, under the assumption that $G_{\hat{T}}^\pi(s, a) \in \FF_1$ and $V^\pi_{T}(s')\in \FF_2$, we have 
\begin{equation}
    \begin{aligned}
    \hat{\eta}[\pi] - \eta[\pi] &= \gamma \mathbb{E}_{(s,a)\sim \rho^\pi_{\hat{T}}}[G_{\hat{T}}^\pi(s, a)] - \gamma \mathbb{E}_{(s,a)\sim \rho^{\pi_D}_{T}}[G_{\hat{T}}^\pi(s, a)] + \gamma \mathbb{E}_{(s,a)\sim \rho^{\pi_D}_{T}}[G_{\hat{T}}^\pi(s, a)] \\
    & \leq \gamma \sup_{f\in \mathcal{F}_1} \left| \mathbb{E}_{(s,a)\sim \rho^{\pi_D}_{T}}[f(s, a)] - \mathbb{E}_{(s,a)\sim \rho^{\pi}_{\hat{T}}}[f(s, a)]  \right|  \\ 
    & +\gamma \mathbb{E}_{(s,a)\sim \rho^{\pi_D}_{T}} \left[\sup_{g\in\mathcal{F}_2}\left| \mathbb{E}_{s' \sim \hat{T}(\cdot|s,a)}[g(s')] - \mathbb{E}_{s' \sim T(\cdot|s,a)}[g(s')] \right|\right] \\
    & = \gamma d_{\mathcal{F}_1}(\rho^{\pi_D}_{T}, \rho^{\pi}_{\hat{T}}) + \gamma \mathbb{E}_{(s,a)\sim \rho^{\pi_D}_{T}}[d_{\mathcal{F}_2}(\hat{T}(\cdot|s,a), T(\cdot|s,a) )].
    \end{aligned}
\end{equation}
By rewriting it as a lower bound form, we have

$$
\eta[\pi] \geq \hat{\eta}[\pi] - \gamma d_{\mathcal{F}_1}(\rho^{\pi_D}_{T}, \rho^{\pi}_{\hat{T}}) - \gamma \mathbb{E}_{(s,a)\sim \rho^{\pi_D}_{T}}[d_{\mathcal{F}_2}(\hat{T}(\cdot|s,a), T(\cdot|s,a) )].
$$

Similarly, if we assume the reward function is bounded,  $d_{\mathcal{F}_2}(\hat{T}, T)$ can also be a total variation distance since $\|V^\pi_T \|_\infty$ is bounded.
By comparing this lower bound to the one in Theorem~\ref{theorem}, it seems this one might be tighter and there is no extra $\epsilon_\pi$ term. But we should notice that the assumptions made here are stronger. To be more specific, we assume $G_{\hat{T}}^\pi$ satisfies the constraint while in Theorem~\ref{theorem} we only assume the model $\hat{T}$ to satisfy the constraint, which is easier to hold.

\end{document}